\documentclass[final,12pt]{clear2024} 

\title[Fairness Explainability using Optimal Transport]{Fairness Explainability using Optimal Transport \\ with Applications in Image Classification}
\usepackage{times}
\usepackage{tikz}
\usepackage{multirow}
\usepackage{dsfont}

\newtheorem{theo}{Theorem}[section]
\newtheorem{rem}[theo]{Remark}
\newtheorem{prop}[theo]{Proposition}

\usepackage{algorithmic}
\usepackage{algorithm}

\usepackage{colortbl}

\newcommand{\argmin}[1]{\underset{#1}{\operatorname{arg}\!\operatorname{min}}\;}

\clearauthor{%
 \Name{Philipp Ratz} \Email{ratz.philipp@courrier.uqam.ca}\\
 \addr Université du Québec à Montréal
 \AND
 \Name{Fran\c{c}ois Hu} \Email{francois.hu@umontreal.ca}\\
 \addr  Université de Montréal
 \AND
 \Name{Arthur Charpentier} \Email{charpentier.arthur@uqam.ca}\\
 \addr Université du Québec à Montréal%
}

\begin{document}

\maketitle

\begin{abstract}%
Ensuring trust and accountability in Artificial Intelligence systems demands explainability of its outcomes. Despite significant progress in Explainable AI, human biases still taint a substantial portion of its training data, raising concerns about unfairness or discriminatory tendencies. Current approaches in the field of Algorithmic Fairness focus on mitigating such biases in the outcomes of a model, but few attempts have been made to try to explain \emph{why} a model is biased. To bridge this gap between the two fields, we propose a comprehensive approach that uses optimal transport theory to uncover the causes of discrimination in Machine Learning applications, with a particular emphasis on image classification. We leverage Wasserstein barycenters to achieve fair predictions and introduce an extension to pinpoint bias-associated regions. This allows us to derive a cohesive system which uses the enforced fairness to measure each features influence \emph{on} the bias. Taking advantage of this interplay of enforcing and explaining fairness, our method hold significant implications for the development of trustworthy and unbiased AI systems, fostering transparency, accountability, and fairness in critical decision-making scenarios across diverse domains.

\end{abstract}


\section{Introduction}

Machine Learning (ML) algorithms are widely used in critical domains ranging from recruiting and law enforcement to personalized medicine~\cite{berk2012criminal, garcia2018proposing, esteva2021deep}. Their usage is not beyond debate however, as fairness related issues remain poorly understood. Further, ML algorithms can perpetuate societal stereotypes and discriminatory practices by associating protected attributes, such as gender and race, with predictions - even if the attribute in question is not directly used in the modelling process~\cite{pedreshi2008discrimination,noiret2021bias, mehrabi2021survey}. This can lead to discriminatory behavior towards certain subgroups, where examples include sexist recruiting algorithm, facial recognition systems that perform poorly for females with darker skin and challenges in recognizing specific subgroups in self-driving cars~\cite{goodall2014ethical, nyholm2016ethics, Dastin_2018}. None of these algorithms were designed with explicit malice, but nevertheless delivered biased results. As \cite{kearns2019ethical} put it, ``\textit{machine learning won’t give you anything like gender neutrality ‘for free’ that you didn’t explicitely ask for}". Whereas there has been notable progress in the elimination of biases from black box models, challenges persist in identifying the source of biases and explaining \emph{why} unfair outcomes materialized~\cite{ali2023explainable}. Especially in fields where complex black-box models are employed, explaining unfairness is often reliant on testing hypotheses one-by-one, which can quickly become infeasible in the era of big-data. A key reason for this is that existing methods aim to explain how a given \emph{score} was constructed not how a \emph{bias} was introduced. Additionally, criticism has been raised regarding the potential misuse of standard explainable AI tools, which can result in misleading explanations that validate incorrect models~\cite{alvarez2018towards, rudin2019stop}. This issue, often referred to as "\textit{fairwashing}"~\cite{aivodji2019fairwashing, aivodji2021characterizing}, underscores the importance of exercising caution in the application of such tools.

To address both the fairness and explainablity concerns, we take a two-fold approach in this article. As a \textbf{First Step} (\textbf{A}), we use existing research on mitigating the impact of sensitive attributes within a \textit{fairness-aware} framework and extend this explicitly to pre-trained models using optimal transport theory~\cite{villani2021topics}. We then turn our attention to explainablity and issues arising from fairwashing in a \textbf{second step} (\textbf{B}). Here, aim to model the algorithmic bias directly, providing valuable insights into the root causes behind the biases. Importantly, our work fulfills both local (specific model outputs) and global (model insights from data) explanation requirements, as highlighted in~\cite{arrieta2020explainable}. The combination of these two steps, (\textbf{A}) and (\textbf{B}), has the advantage that algorithmic fairness can be enforced and steps to mitigate the source of the biases can be put into place. By filling this important gap, we contribute to enhancing fairness, transparency and accountability in Artificial Intelligence (AI) systems. Note that throughout this article, we focus on applications involving images, rather than traditional tabular data, as it allows us to showcase the effectiveness of our approach. Specifically, with images, a more direct interpretation is possible, even without domain knowledge. 
However, the techniques presented here can easily be applied to standard tabular data sets as well.

\subsection{Scoping and Definitions}

The field of Algorithmic Fairness considers different metrics for distinct goals. Here we opt to consider fairness at the distributional level, specifically, we focus on the Demographic Parity (DP) notion of fairness~\cite{calders2009building}. This strict definition aims to achieve independence between sensitive attributes and predictions without relying on labels. Formally, let $(\boldsymbol{X}, S, Y)$ be a random tuple with distribution $\mathbb{P}$, where $\boldsymbol{X}\in\mathcal{X}\subset \mathbb{R}^d$ represents the features, $S\in\mathcal{S}\subset \mathbb{N}$ a sensitive feature, considered discrete, across which we would like to impose fairness and $Y\in\mathcal{Y} := \{0, 1\}$ represents the task to be estimated. As an illustration, consider Figure \ref{fig:tupleviz}. 
Our study primarily focuses on binary classification tasks; however, the methodologies and techniques discussed can be readily extended and generalized to regression tasks or multi-task problems. Also note that we 
include the sensitive variable $S$ in the model, which is a somewhat paradoxical feature in Algorithmic Fairness. 
However, both empirical studies \cite{lipton2018does} and theoretical research \cite{gaucher2023fair} have consistently shown that de-biasing algorithms that do not consider the sensitive attribute, referred to as \textit{fairness-unaware}, exhibit inferior fairness compared to \textit{fairness-aware} approaches \cite{dwork2012fairness} that leverage the sensitive feature. 

\begin{figure}[ht]
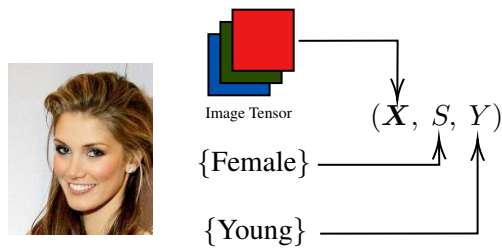

	\vspace{.1in}
	\centering
    \include{tikz/explainer_attributes}
    \caption{
    Attributes of an Image, here we consider the image to be the features (that is, $\boldsymbol{X}$) and the sensitive attribute (here, gender) as well as the label (young) a categorical variable. 
    }
    \label{fig:tupleviz}
\end{figure}

In binary classification, we aim to determine the probability response to get $1$ in $\mathcal{Z}:= [0, 1]$ for classifying $(\boldsymbol{X}, S)$, also known as a \textit{probabilistic classifier} (or \textit{soft} classifier). To achieve strong DP-fairness for a given predictor $f$, the objective is to ensure that $S$ and $f(\boldsymbol{X}, S)$ are independent. This definition is more flexible than its commonly used \textit{weak} counterpart. Weak DP-fairness, aims to establish the independence of $S$ and the \textit{hard} classifier $c_f (\boldsymbol{X}, S):= \mathds{1}\{f (\boldsymbol{X}, S) \geq 0.5\}$ but restricts the user to a given threshold. Throughout our study, our goal is to satisfy the strong DP-fairness, formally defined as the following conditions for all $s, s'\in\mathcal{S}$ and $u\in\mathbb{R}$:

\begin{equation}\label{eq:IntroDP}
    \mathbb{P}\left(f(\boldsymbol{X}, S) \leq u | S = s \right) = \mathbb{P}\left(f(\boldsymbol{X}, S) \leq u | S = s' \right) \enspace.
\end{equation}

Further, we define $\mathcal{F}$ as the set of soft classifiers of the form $f:\mathcal{X}\times\mathcal{S}\to \mathcal{Z}$. Also, given $s\in\mathcal{S}$ we denote,
\begin{itemize}
    \item $\nu_{f}$ (resp. $\nu_{f|s}$) the probability measure of $f(\boldsymbol{X}, S)$ (resp. $f(\boldsymbol{X}, S)|S=s$);

    \item $F_{f|s}(u) := \mathbb{P}\left( f(\boldsymbol{X}, S) \leq u|S=s \right)$ its cumulative distribution function (CDF);

    \item $Q_{f|s}(v) := \inf\{u\in\mathbb{R}:F_{f|s}(u)\geq v \}$ the associated quantile function.
\end{itemize}
Throughout the remainder of the paper, we assume that for any $f\in\mathcal{F}$, both the measures $\nu_f$ and $\nu_{f|s}$ have finite second-order moments and their densities exist. With these notations, the DP notion of fairness defined in Equation~\eqref{eq:IntroDP} is rewritten as $F_{f|s}(u) = F_{f|s'}(u)$ for all $s, s'\in\mathcal{S}$ and $u\in\mathbb{R}$.

\subsection{Related Work}

\subsubsection{Algorithmic Fairness}

In recent years, research on algorithmic fairness has grown significantly. The most common approaches can broadly be categorized into pre-processing, in-processing, and post-processing methods. Pre-processing ensures fairness in the input data by removing biases before applying ML models \cite{park2021learning,qiang2022counterfactual}. In-processing methods incorporate fairness constraints during model training  \cite{wang2020towards, joo2020gender}, where fairness constraints usually modify the loss landscape and prevent the model from learning an unfair solution. Whereas these two approaches focus on the model parameters itself, post-processing techniques aim to achieve fairness through modifications of the final scores \cite{karako2018using, kim2019multiaccuracy}. This has the advantage that it works with any kind of estimator, including (partially) pre-trained models. Especially in computationally intensive fields, transfer learning or partial fine-tuning are prevalent to reduce training time and improve generalization. Due to this, post-processing methods are easily integrable into a standard workflow at low computational cost, hence we focus our work to this latter category.
Of particular importance is the literature using Optimal Transport, a mathematical framework for measuring distributional differences. Intuitively, the goal is to \emph{transport} unfair scores to fair ones while minimizing the effects of this intervention to maintain predictive accuracy. In regression, methods like \cite{Chzhen_Denis_Hebiri_Oneto_Pontil20Wasser} and \cite{gouic2020projection} minimize Wasserstein distance to reduce discrimination. Similarly, in classification, \cite{chiappa2020general} and \cite{gaucher2023fair} leverage optimal transport to achieve fair scores. Recent work by \cite{hu2023fairness} also achieves fairness in multi-task learning through joint optimization. Though still a nascent field, applications thereof become more common \cite{zehlike2020matching,charpentier2023mitigating}. However, despite the extensive use of optimal transport theory in algorithmic fairness, there is only limited research that delves into applications that go beyond the standard case of tabular data, a first shortcoming we address in this article. 

\subsubsection{Explainable AI}
We focus on creating a simple fairness explanation method for computer vision, narrowing down our exploration to two Explainable AI (XAI) subfields.
Among the most widely known methods are \emph{model-agnostic} techniques, such as LIME~\cite{ribeiro2016should, garreau2021does} and SHAP~\cite{shapley1997value, lundberg2017unified}, that do not depend on specific assumptions about the model's architecture. Though these approaches can be extended to the analysis of images, a range of XAI methods have been developed more specifically for the use of deep neural networks. These methods commonly focus on local explainability, which often involves highlighting important pixels (referred to as \emph{attention maps}) for individual task predictions. Global explainablity can then be achieved through the identification of significant regions across the whole prediction analysis. As these approaches leverage the specific architecture of neural networks they are referred to as \emph{model-specific} approaches. Notable examples include Grad-CAM~\cite{selvaraju2017grad} and its various variants, like Grad-CAM++ \cite{chattopadhay2018grad} and Score-CAM \cite{wang2020score}. Recent work by \cite{franco2021toward} for fair and explainable systems, generating two attention maps for local insight. Global explainability is achieved through t-SNE representations, but explicit discrimination explanations are often lacking, raising potential \emph{fairwashing} concerns~\cite{alikhademi2021can}. Our approach sets itself apart from the aforementioned methods by directly generating attention maps that specifically describe the model's unfair decisions, offering a clearer and more focused explanation for discriminatory outcomes.

\subsection{Contributions and Outline}
In summary, we extend the literature of fairness-aware algorithms based on optimal transport theory through the following points:
\begin{itemize}
    \item \textbf{Fair decision-making}:
    We adapt a post-processing model using optimal transport theory for computer vision tasks, bringing the theory closer to the community. We ensure fair and unbiased outcomes and show that the solution is optimal with respect to the relative rankings, and independent of the bias. 
    \item \textbf{Explainable artificial intelligence}: 
    Our main contribution is to use the optimal transport plan to develop an XAI approach for identifying changes in the data, describing unfair results. In computer vision applications, our method directly highlights the regions most responsible for the stated bias, facilitating direct identification of discrimination. The method is also easily extendable to tabular data sets. 
\end{itemize}
The remainder of this article is structured as follows. First, we provide a brief background on optimal transport theory and establish its connection to algorithmic fairness. Then, we turn our focus to our XAI methodologies to uncover the causes of discrimination.
Finally, we showcase their performance through numerical experiments on the \texttt{CelebA} dataset.

\section{Background on Optimal Transport}

In this section, we present the fundamental concepts from optimal transportation theory. Specifically, we focus on the Wasserstein distance and give a brief overview of notable results in optimal transport theory with one-dimensional measures, where all the main results can be found in \cite{villani2009optimal, santambrogio2015optimal, villani2021topics}.

\subsection{Wasserstein Distances}

Let $\nu_{f_1}$ and $\nu_{f_2}$ be two probability measures on $\mathcal{Z}$. The squared Wasserstein distance (cf. \cite{santambrogio2015optimal}, definition \S5.5.1) between $\nu_{f_1}$ and $\nu_{f_2}$ is defined as
\begin{equation*}
    \mathcal{W}_2^2(\nu_{f_1}, \nu_{f_2}) = \inf_{\pi\in\Pi(\nu_{f_1}, \nu_{f_2})} \mathbb{E}_{(Z_1, Z_2)\sim \pi}\left(Z_2-Z_1\right)^2 \enspace,
\end{equation*}
where $\Pi(\nu_{f_1}, \nu_{f_2})$ is the set of distributions on $\mathcal{Z}\times\mathcal{Z}$ having $\nu_{f_1}$ and $\nu_{f_2}$ as marginals. If the infimum is achieved, the resulting coupling is referred to as the optimal coupling between $\nu_{f_1}$ and $\nu_{f_2}$. If either of the predictors belongs to $\mathcal{F}$, the optimal coupling can be determined (refer to \cite{villani2021topics}, Thm 2.12) as follows: if $Z_1\sim\nu_{f_1}$ and $Z_2\sim\nu_{f_2}$, where $f_2\in\mathcal{F}$, there exists a mapping $T : \mathbb{R}\to \mathbb{R}$ such that
$$
\mathcal{W}_2^2(\nu_{f_1}, \nu_{f_2}) = \mathbb{E}\left( Z_2 - T(Z_2)\right)^2\enspace,
$$
with $T(Z_2)\sim \nu_{f_1}$.
We call $T$ the optimal transport map from $\nu_{f_2}$ to $\nu_{f_1}$. Moreover, in the univariate setting, a closed-form solution is explicitly provided as: $T(\cdot) = Q_{f_1}\circ F_{f_2} (\cdot)$.

\subsection{Wasserstein Barycenters}

Throughout this article, we will frequently make use of \textit{Wasserstein Barycenters}. It can be defined for a given family of $K$ measures $(\nu_{f_1}, \dots, \nu_{f_K})$ and weights $\boldsymbol{w} = (w_1, \dots, w_K) \in\mathbb{R}_+^K$ such that $\sum_{s=1}^Kw_s = 1$. Then $Bar(w_s, \nu_{f_s})_{s=1}^K$ represents the Wasserstein barycenter of these measures which is the minimizer of
\begin{equation*}
    Bar(w_s, \nu_{f_s})_{s=1}^K = \argmin{\nu}\sum_{s=1}^{K}w_s\cdot \mathcal{W}_2^2\left( \nu_{f_s}, \nu\right)\enspace.
\end{equation*}
The work in \cite{agueh2011barycenters} shows that in our configuration, with $f_s\in\mathcal{F}$, this barycenter exists and is unique. Put into words, the Wasserstein barycenter can be used to find a representative distribution that lies between multiple given distributions. For our applications, this will ensure that the predictive distributions given any values $s$ will coincide. The optimal transport problem then aims to minimize the total amount of changes required to achieve this.

\section{Fairness Projection using Optimal Transport}\label{sec:Fair} 

Optimal transport theory provides a means to ensure specific forms of algorithmic fairness. We provide a short summary of the some necessary concepts, all the main results can be found in~\cite{chiappa2020general, Chzhen_Denis_Hebiri_Oneto_Pontil20Wasser} and~\cite{gouic2020projection}.

\subsection{Unfairness and Risk: a Warm-up} 
Paraphrasing the quote above that machine learning will not give a fair classifier for free, we first need to define both the objective of the classification and the concept of DP in a unified manner. DP will be used to determine the fairness of a classifier. 
\begin{definition}[Demographic Parity] Given a soft classifier $f$, its unfairness is quantified by
    \begin{equation}\label{eq:Unf}
        \mathcal{U}(f) = \max_{s, s'\in\mathcal{S}}\sup_{u\in\mathcal{Z}}\left| F_{f|s}(u) - F_{f|s'}(u) \right|\enspace,
    \end{equation}
and $f$ is called (DP-)fair if and only if $\mathcal{U}(f)=0$.
\end{definition}
Consider $f^*(\boldsymbol{X}, S) := \mathbb{E}\left[Y|\boldsymbol{X}, S\right]$, the Bayes rule that minimizes the following squared risk 
$$
\mathcal{R}(f) := \mathbb{E} \left( Y - f(\boldsymbol{X}, S) \right)^2\enspace.
$$ 
The associated \textit{hard} classifier $c_{f^*}(\boldsymbol{X}, S)$ has the property of minimizing the risk of misclassification, which makes the squared risk applicable for both regression and classification (for more details, see \cite{gaucher2023fair}).
In line with this, we adopt a popular approach to algorithmic fairness by incorporating DP-fairness principles into risk minimization~\cite{Chzhen_Denis_Hebiri_Oneto_Pontil20Wasser, gouic2020projection}, that is:
\begin{equation*}
    \min_{f\in\mathcal{F}} \left\{\mathcal{R}(f) : \mathcal{U}(f) = 0\right\}\enspace.
\end{equation*}
By construction, this optimization effectively balances both risk $\mathcal{R}$ and unfairness $\mathcal{U}$, leading to improved predictive performance, reduced biases, and mitigation of potentially offensive or discriminatory errors. 

\subsection{Optimal Fair Projection: Theoretical and Empirical Estimators}
\label{subsec:Estimators}
The two objectives, fairness and predictive accuracy are often in conflict with one another. Most of the recent work in algorithmic fairness has therefore focused on finding either a precise joint solution or an optimal trade-off between the two. When starting from the best possible predictor without any constraints, we refer to its optimal \emph{fair} counterpart as the \emph{optimal fair projection}. This estimator should minimize the unfairness across the sensitive variables, while maintaining the best predictive accuracy under this constraint. Much work has been done within the field to achieve univariate optimal fair projections. Recall $p_s = \mathbb{P}(S=s)$ and let $f_B\in\mathcal{F}$, where its measure is the Wasserstein barycenter $\nu_{f_B} := Bar(p_s, \nu_{f^*|s})_{s\in\mathcal{S}}$. Then, studies conducted by \cite{Chzhen_Denis_Hebiri_Oneto_Pontil20Wasser} and \cite{gouic2020projection} demonstrate that
\begin{equation*}
    f_B = \argmin{f\in\mathcal{F}} \lbrace \mathcal{R}(f) : \mathcal{U}(f) = 0\rbrace\enspace.
\end{equation*}
Therefore, $f_B$ represents the optimal fair predictor in terms of minimizing unfairness-risk. Previous studies also offer a precise closed-form solution: for all $(\boldsymbol{x}, s) \in \mathcal{X}\times \mathcal{S}$,
\begin{equation}\label{eq:OptFair}
    f_B(\boldsymbol{x}, s) = \left( \sum_{s'=1}^{K} p_{s'}Q_{f^*|s'}\right)\circ F_{f^*|s}\left( f^*(\boldsymbol{x}, s) \right)\enspace.
\end{equation}
To employ the results on real data the plug-in estimator of the Bayes rule $f^*$ is given by $\hat{f}$, which corresponds to any DP-unconstrained ML model trained on a training set $\{(\boldsymbol{x}_i, s_i, y_i)\}_{i=1}^{n}$ $n$ i.i.d. realizations of $(\boldsymbol{X}, S, Y)$. The empirical counterpart is then defined as:
\begin{equation}\label{eq:plugin}
  \widehat{f_B}(\boldsymbol{x}, s) = \left( \sum_{s'=1}^{K} \Hat{p}_{s'}\Hat{Q}_{\Hat{f}|s'}\right)\circ \Hat{F}_{\Hat{f}|s}\left( \Hat{f}(\boldsymbol{x}, s) \right)\enspace,  
\end{equation}
where $\Hat{p}_s$, $\Hat{F}_{\hat{f}|s}$ and $\Hat{Q}_{\hat{f}|s}$ corresponds to the empirical counterparts of $p_s$, $F_{f^*|s}$ and $Q_{f^*|s}$. Note that, with the exception of $\hat{f}$, the remaining quantities can be constructed using an unlabeled \emph{calibration} dataset, denoted as $\mathcal{D}^{\text{pool}} := \{(\boldsymbol{X}_i, S_i)\}_{i=1}^N$, which consists of N i.i.d. copies of $(\boldsymbol{X}, S)$. The pseudo-code of this approach is provided in Algorithm \ref{alg:projection}. We also visualize a possible model flow in Figure \ref{fig:model_arch}, where the \textit{calibration layer} corresponds to the inner workings of Algorithm \ref{alg:projection} and specifically Equation \ref{eq:plugin}. \cite{Chzhen_Denis_Hebiri_Oneto_Pontil20Wasser} show that if the estimator $\Hat{f}$ is a good proxy for $f^*$, then under mild assumptions on the distribution $\mathbb{P}$, the calibrated post-processing approach $\widehat{f_B}$ is a good estimator of $f_B$, enabling accurate and fair estimation of the instances. \cite{gaucher2023fair} demonstrates that the hard classifier $c_{f_B}$ maximizes accuracy under DP-constraint, and the classifier $c_{\widehat{f_{B}}}$ is proven to be a good estimator.

\begin{figure}[ht]
  \centering
  \begin{minipage}{.7\linewidth}
\begin{algorithm}[H]
    \footnotesize
   \caption{Fairness projection.}
   \label{alg:projection}
\begin{algorithmic}
   \STATE {\bfseries Input:} instance $(\boldsymbol{x}, {s})$, base estimator $\Hat{f}$, unlabeled data $\mathcal{D}^{\text{pool}}  = \{(\boldsymbol{x}_i, s_i)\}_{i=1}^N$.
   \vspace{0.15cm}
   \STATE {\bf \quad Step 0.} Split $\mathcal{D}^{\text{pool}}$ to construct the group-wise sample
   $$
   \{\boldsymbol{x}_{i}^s\}_{i=1}^{N_s}\sim \mathbb{P}_{\boldsymbol{X}|S=s} \quad\text{for any } s\in \mathcal{S}\enspace;
   $$
   \quad with $N_s$ the number of images corresponding to $S=s$
   \vspace{0.1cm}
   \STATE {\bf \quad Step 1.} Compute the frequencies $(\widehat{p}_s)_s$ from $\{s_i\}_{i=1}^N$;
   \vspace{0.1cm}
   \STATE {\bf \quad Step 2.} Estimate $(\hat{F}_{\hat{f}|s})_s$ and $(\hat{Q}_{\hat{f}|s})_s$ from $\{\boldsymbol{x}_{i}^s\}_{i=1}^{N_s}$;
   \vspace{0.1cm}
   \STATE {\bf \quad Step 3.} Compute $\widehat{f}_B$ according to Eq.~\eqref{eq:plugin};
   \vspace{0.15cm}
   \STATE {\bfseries Output:} fair predictor $\widehat{f}_B(\boldsymbol{x},s)$ at point $(\boldsymbol{x},s)$.
\end{algorithmic}
\end{algorithm}
\end{minipage}
\end{figure}

\section{Explainable AI using the Transportation Plan}\label{sec:XAIImage}

Whereas the results from above enable the correction of given scores, they can be considered a treatment of the symptoms rather than the cause of unfairness. 
To this end, we extend the fairness procedures from above to explicitly pinpoint these sources bias. The key idea that we pursue is that the transport map used in Equation \eqref{eq:plugin} can be used to construct group-wise counterfactual estimates. 
This approach clarifies differences between pre- and post-processed scores, enabling the use of established XAI methods to uncover the underlying causes of unfairness.

\subsection{Local explainability}

To isolate the features responsible for discrimination, we extend the fair projection method by introducing an auxiliary learning task to detect the source of biases directly. The binary task, denoted $\Tilde{Y}\in\Tilde{\mathcal{Y}}:=\{0, 1\}$, is estimated using the distributions of the unfair predictor $f^*(\boldsymbol{X}, S)$ and the DP-fair predictor $f_B^*(\boldsymbol{X}, S)$. This task can be chosen according to specific goals, and we provide a sample of possibilities in Table~\ref{tab:BiasDetectTask}. The formulation
$$
d_B(\boldsymbol{X}, S) := f_B^*(\boldsymbol{X}, S) - f^*(\boldsymbol{X}, S)\enspace,
$$
then offers an intuitive explanation of unfairness. As an example, in the context of a wage model, a positive value of $d_B(\boldsymbol{X}, S)$ indicates a group discrimination for an individual $(\boldsymbol{X}, S)$ while a negative value might indicate favoritism. Its magnitude $|d_B(\boldsymbol{X}, S)|$ can be interpreted as the ``degree" of discrimination or favoritism, and its squared version an indicator for extremes. The proposition below provides an interpretation of the quantity $d_B(\boldsymbol{X}, S)$ within the probabilistic framework, specifically in the context of a binary sensitive attribute scenario.
\begin{prop}[Bias detection characterization] 
\label{prop:BiasDetection}
Suppose $\mathcal{S}=\{1, 2\}$ a binary sensitive feature scenario. Given $(\boldsymbol{x}, s)\in\mathcal{X}\times S$ and $\Bar{s}\in \mathcal{S}-\{s\}$, there exists an optimal transport map from $\nu_{f^*|s}$ to $\nu_{f^*|\Bar{s}}$, denoted $T_{s\to \Bar{s}} : \mathcal{Z}\to \mathcal{Z}$, such that $d_B(\boldsymbol{x}, s) = f^*_B(\boldsymbol{x}, s) - f^*(\boldsymbol{x}, s)$ can be rewritten as,
    \begin{equation}
    \label{eq:BiasDetectChara}
        d_B(\boldsymbol{x}, s) = p_{\Bar{s}}\cdot\left(\ T_{s\to\Bar{s}} \circ f^*(\boldsymbol{x}, s) - f^*(\boldsymbol{x}, s)\ \right)\enspace,
    \end{equation}
    where $T_{s\to\Bar{s}} \circ f^*(\boldsymbol{X}, s)\sim \nu_{f^*|\Bar{s}}$.
\end{prop}

\begin{proof} 
     Given $(\boldsymbol{x}, s)\in\mathcal{X}\times S$, we are interested in the quantity $d_B(\boldsymbol{x}, s) = f^*_B(\boldsymbol{x}, s) - f^*(\boldsymbol{x}, s)$. 
     For simplicity, we denote $u_s(\boldsymbol{x}) = F_{f^*|s} (f^*(\boldsymbol{x}, s))$. 
     Then, given $\Bar{s}\in \mathcal{S}-\{s\}$, we have,
    \begin{align*}
        f^*_B(\boldsymbol{x}, s) - f^*(\boldsymbol{x}, s) & = \left( \sum_{s'=1, 2} p_{s'}Q_{f^*|s'}\right)\circ F_{f^*|s}\left( f^*(\boldsymbol{x}, s) \right) - f^*(\boldsymbol{x}, s) \\
         & = p_{\Bar{s}}\cdot Q_{f^*|\Bar{s}} (u_{s}(\boldsymbol{x})) - 
         p_{s}\cdot Q_{f^*|s} (u_{s}(\boldsymbol{x})) + Q_{f^*|s} (u_{s}(\boldsymbol{x}))\\
         & = p_{\Bar{s}}\cdot Q_{f^*|\Bar{s}} (u_{s}(\boldsymbol{x})) - 
         p_{\Bar{s}}\cdot Q_{f^*|s} (u_{s}(\boldsymbol{x}))\\
         & = p_{\Bar{s}}\cdot \left( Q_{f^*|\Bar{s}} (u_{s}(\boldsymbol{x})) -  f^*(\boldsymbol{x}, s)\right)\enspace.
    \end{align*}
    Since $f^*\in\mathcal{F}$, by definition of the Wasserstein distance, there exists a transport map $T_{s\to\Bar{s}}:\mathcal{Z} \to \mathcal{Z}$ such that $T_{s\to\Bar{s}}(\cdot) = Q_{f^*|\Bar{s}}\circ F_{f^*|s} (\cdot)$ with $T_{s\to\Bar{s}}\circ f^*(\boldsymbol{X}, s) \sim \nu_{f^*|\Bar{s}}$, which concludes the proof.
\end{proof}


In a binary sensitive framework, Proposition~\ref{prop:BiasDetection} asserts that $|d_B(\boldsymbol{X}, s)|$ depends on how much the DP-unconstrained prediction for $(\boldsymbol{X}, s)$ deviates from its projection onto $\nu_{f^*|\Bar{s}}$. In other words, the \textit{r.h.s.} of Equation~\eqref{eq:BiasDetectChara} measures the disparity between the initial prediction and the projected prediction, where features $\boldsymbol{X}|S=s$ are aligned with $\boldsymbol{X}|S=\Bar{s}$. As a concrete example, changing a male individuals' gender to female, the projection also modifies related attributes (such as height or weight) to match a female counterpart, ensuring comparability when these attributes naturally differ on a group level. 
Note that $p_{\Bar{s}}$ enhances this bias for over-represented $\Bar{s}$ groups, but reduces its significance for under-represented ones.



\begin{table}[htbp]
    \footnotesize
    \begin{center}
    \begin{tabular}{|c|c|c|}
    \hline
     \textbf{\textit{ Task description }} & \textbf{\textit{ Probabilistic framework}} & \textbf{\textit{Empirical framework}} \\
    \hline
    Discrimination & $\mathds{1}\lbrace f_B^*(\boldsymbol{X}, S) - f^*(\boldsymbol{X}, S) \geq 0\rbrace$ & $\mathds{1}\lbrace\widehat{f_B}(\boldsymbol{X}, S) - \widehat{f}(\boldsymbol{X}, S) \geq 0\rbrace$ \\
    \hline
     Bias size & $\mathds{1}\lbrace|f_B^*(\boldsymbol{X}, S) - f^*(\boldsymbol{X}, S)| \geq \tau\rbrace$ & $\mathds{1}\lbrace|\widehat{f_B}(\boldsymbol{X}, S) - \widehat{f}(\boldsymbol{X}, S)| \geq \tau\rbrace$ \\
    \hline
     Outliers & $\mathds{1}\lbrace(f_B^*(\boldsymbol{X}, S) - f^*(\boldsymbol{X}, S))^2 \geq \tau\rbrace$ & $\mathds{1}\lbrace(\widehat{f_B}(\boldsymbol{X}, S) - \widehat{f}(\boldsymbol{X}, S))^2 \geq \tau\rbrace$ \\
    \hline
    \end{tabular}
    \caption{Bias / discrimination detection task}
    \label{tab:BiasDetectTask}
    \end{center}
\end{table}

Alternatively, the new task described in Equation~\eqref{eq:BiasDetectChara} can be viewed as a decomposition of biases into implicit and explicit components. Indeed, applying triangle inequality we have
    \begin{align*}
        |d_B(\boldsymbol{x}, s)| & \leq p_{\Bar{s}}\cdot\left( \ \left|T_{s\to\Bar{s}} \circ f^*(\boldsymbol{x}, s) - f^*(\boldsymbol{x}, \Bar{s})\right| +  \left|f^*(\boldsymbol{x}, \Bar{s}) - f^*(\boldsymbol{x}, s)\right|\ \right)\enspace.
    \end{align*}
In this context, implicit bias, which refers to the hidden influence of $s$, is measured as the difference between two values $T_{s\to\Bar{s}} \circ f^*(\boldsymbol{X}, s) - f^*(\boldsymbol{X}, \Bar{s})$ 
both follow the same distribution $\nu_{f^*|\Bar{s}}$ (although not independent). This implicit bias represents the variation in predicted outcomes when the features $\boldsymbol{X}$ under $S=s$ are aligned with those under $S=\Bar{s}$, in contrast to unconstrained predictions where $s$ is simply replaced with $\Bar{s}$ without aligning the features. Explicit bias, conversely, is simplistically expressed as $f^*(\boldsymbol{X}, \Bar{s}) - f^*(\boldsymbol{X}, s)$.
This measurement aligns to the principle of \emph{ceteris paribus}, meaning ``\emph{all other things being equal}". However, when considered in isolation, this condition can lead to unrealistic situations. We provide a visual explanation in the Appendix \ref{app:vizref}.


\paragraph{Data-driven procedure} In real datasets, we use plug-in estimators from Section~\ref{subsec:Estimators} to estimate \(f^*\) and $f_B^*$, producing $\widehat{d_B} = \widehat{f_B} - \widehat{f}$, the empirical counterpart of $d_B$. Our goal is to train an estimator $g : \mathcal{X} \to \Tilde{\mathcal{Y}}$, where $\Tilde{Y} \in \Tilde{\mathcal{Y}}$ represents the new target task outlined in Table~\ref{tab:BiasDetectTask}'s last column. XAI methods are then used to pinpoint areas causing observed model unfairness in the initial ML model. For image classification, popular techniques like Grad-CAM (\cite{selvaraju2017grad}) create attention maps highlighting these biased areas. The pseudo-code for this approach is provided in Algorithm~\ref{alg:LocalXAI} with $\Tilde{Y}^\tau:= \mathds{1}\lbrace|\widehat{f_B}(\boldsymbol{X}, S) - \widehat{f}(\boldsymbol{X}, S)| \geq \tau\rbrace$ as the desired XAI task.

\begin{rem}[Impact of the parameter $\tau$ on the bias detection]
    In Table~\ref{tab:BiasDetectTask}, various tasks require establishing a threshold \(\tau > 0\) to identify essential bias-contributing regions. 
    We suggest determining \(\tau\) at a specific quantile \(\alpha \in (0, 1)\) within the sample \(\{|\widehat{d_B}(\boldsymbol{x_i}, s_i)|\}_{1\leq i\leq N}\), denoted as \(\widehat{Q}_{|\widehat{d_B}|}(\alpha)\). In particular, the choice of \(\alpha\) significantly influences the behavior of the bias detector. Indeed, a larger \(\alpha\) emphasizes causes with very high biases, while a smaller value identifies all possible causes the bias detector can identify. Opting for \(\alpha = 0.75\) might be a good choice since it results in a more balanced dataset, with approximately equal occurrences of \(\Tilde{Y}^\tau = 0\) and \(\Tilde{Y}^\tau = 1\), while also distinguishing significant areas for unfairness.
\end{rem}

\begin{algorithm}
    \footnotesize
   \caption{Bias Detection}
   \label{alg:LocalXAI}
\begin{algorithmic}
   \STATE {\bfseries Input:} new instance $(\boldsymbol{x}, s)$, base estimator $\Hat{f}$, fair estimator $\widehat{f_B}$, unlabeled sample $\mathcal{D}^{\text{pool}}  = \{(\boldsymbol{x}_i, s_i)\}_{i=1}^N$.
   \vspace{0.15cm}
   \STATE {\bf \quad Step 0.} Generate $(\Tilde{y}^\tau_i)_{1\leq i \leq N}$, where $\Tilde{y}^\tau$ can be constructed according to Table \ref{tab:BiasDetectTask} or according to the objective
   
   \vspace{0.11cm}
   \STATE {\bf \quad Step 1.} Train a learning model $g$ on $\{(\boldsymbol{x}_i, \Tilde{y}^\tau_i)\}_{i=1}^N$;
   \vspace{0.1cm}
   \STATE {\bf \quad Step 2.} Use Grad-CAM (or other XAI method) to generate attention maps;
   \vspace{0.15cm}
   \STATE {\bfseries Output:} Attention map of $g$ on instance $(\boldsymbol{x}, s)$.
\end{algorithmic}
\end{algorithm}

\section{Experiments}

We opt to showcase our method on image data, rather than tabular data, as it helps to highlight an important practical aspect. 
Computer vision tasks are often performed on (partially) pre-trained models and compute time presents a major issue. The post-processing approach outlined in Equation \eqref{eq:plugin} is particularly attractive in these circumstances. As an example, rendering the scores fair throughout the applications in the experimental section took less than 0.1 seconds on average. Further, the XAI approach outlined in the previous section also works on pre-trained models as the transportation plan only depends on the produced scores. We first present how the standard approach outlined in Section \ref{sec:Fair} can be extended to standard computer vision architectures and then show how the bias detection task can help identify the regions associated with the bias. 




\subsection{Extension to Image Classification}\label{sec:FairImage}

An important detail from the above section is that in order to eliminate a bias from the predictions, the sensitive feature $S$ must be included in the modelling process to satisfy the assumptions of the optimal transport theory. This is due to the fact that simply excluding the sensitive information might lead the model to proxy for the sensitive variable which leads back to the initial problem of the biased predictions. For image classification, we consider a standard split into a feature (or embedding) and classification block, where we use a pre-trained embedding model and keep its parameters fixed throughout the whole procedure. The sensitive feature can then be added through a simple layer concatenation of the output from the embedding before it is fed into a classification block. As this will still result in biased scores, a supplementary calibration layer is added, which implements Equation \eqref{eq:plugin} and needs to be trained separately from the main model on a calibration data set. This indicates the need to split the data into three separate parts, the train set to fit the classification block to the specific data, a calibration set (corresponding to $\mathcal{D}^{\text{pool}}$ in Algorithm \ref{alg:projection}, that does \emph{not} need to be labeled), and a standard test set. The architecture is visualized in Figure \ref{fig:model_arch}. 

\begin{figure}[ht]
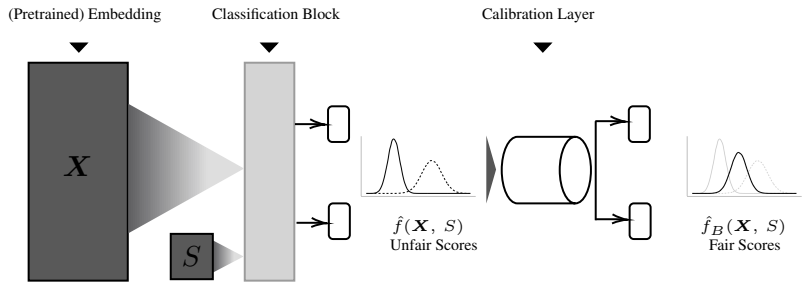

	\vspace{.1in}
	\centering
    \include{tikz/tikz_explainer_2}
    \caption{Standard fairness approach, the sensitive variable $S$ must be included the latest in the classification block. The calibration layer contains an operation which makes use of the Wasserstein Barycenter. The illustration above \emph{Unfair Scores} represent the marginal score distributions induced by a binary sensitive variable which are then transported to a common distribution for the \emph{Fair Scores}.
    }
    \label{fig:model_arch}
\end{figure}




\subsection{Dataset and model}
For our illustrations, we use the \texttt{CelebA}~\cite{liu2015deep} data set, containing more than 200,000 celebrity images, each annotated with 40 attributes. For the visualisations we selected a subset of 5,000 images with well-aligned facial features to obtain averaged predictions. Following the setup from Figure \ref{fig:model_arch}, we choose the pre-trained version of torchvision's \texttt{IMAGENET1K\_V2} trained on the imagenet database as our embedding model. With it, we use the built-in preprocessing steps on each image but consider the parameters fixed. On top of the embedding, we then add a classification block, similar in spirit to what \cite{wang2020towards} proposed. This block contains, before the output layer, three layers of size $[512,256,32]$ which take as input the vector of size 2048 from the embedding block and the sensitive feature of size 1. Each intermediate layer has a ReLU activation function applied to it and uses a 0.1 dropout. We split the data into 64\% training, 16\% calibration and 20\% test data. The model is then trained task wise for 10 epochs using a binary cross entropy loss and average results over 10 runs\footnote{All code can be found at: \href{https://github.com/FairInterpret/fair_images}{github.com/FairInterpret/fair\_images}}.

\subsection{Metrics and Prediction Tasks}
As the approach is valid for soft classifiers, we use the Area under the ROC curve (AUC) as the performance metric on the test set (denoted $A(f)$). We also measure the unfairness $\widehat{\mathcal{U}}(f)$ on the test-set as the empirical counterpart of Equation~\eqref{eq:Unf}, based on the Kolmogorov–Smirnov test,
$$
\widehat{\mathcal{U}}(f) := \max_{s, s'\in\mathcal{S}} \sup_{t\in\mathcal{Z}} \left| \Hat{F}_{f|s}(t) - \Hat{F}_{f|s'}(t) \right| \enspace,
$$
where $\Hat{F}_{f|s}$ is the empirical CDF of $f(\boldsymbol{X}, S)|S=s$. 

We consider three different binary prediction tasks from the data set (the variables \emph{Attractive}, \emph{Beard} and \emph{Young}) and consider \emph{Gender} as the sensitive variable. As bias identification usually requires substantial domain expertise, we demonstrate how our method works when we isolate the \emph{Beard} prediction task from influences of \emph{Gender}. As the positive labels for the task are almost exclusively present for male instances, we would expect the bias to be the largest for this task. Further, the task also has a well defined region of the image that should be used by the model to predict the label, making it suitable for visualisations with Grad-CAM.

\subsection{Results}

\begin{table}[htbp]
\footnotesize
\begin{center}
\begin{tabular}{|c|c|c|c|c|}
\hline
\multicolumn{1}{|c|}{\textbf{}} & \multicolumn{2}{|c|}{\textbf{Uncalibrated}} & \multicolumn{2}{|c|}{\textbf{Fairness-aware}} \\
\cline{1-5}
\cline{0-3}
 \textbf{\textit{Metric}}& {$A(f)$}& \textbf{$\hat{\mathcal{U}}(f)$}& {$A(f)$}& \textbf{$\hat{\mathcal{U}}(f)$} \\
 \hline
 \hline
Attractive & 0.855$\pm$0.002 & 0.447$\pm$ 0.028 & 0.769$\pm$ 0.002 & \cellcolor{blue!15}0.011$\pm$ 0.001 \\
\hline

    Beard & 0.941$\pm$ 0.002 & 0.896 $\pm$ 0.009 & 0.731 $\pm$ 0.004 & \cellcolor{blue!15}0.010$\pm$ 0.002  \\

\hline
Young & 0.858$\pm$0.003 & 0.323$\pm$ 0.036 & 0.81$\pm$ 0.003 & \cellcolor{blue!15}0.013$\pm$ 0.003 \\
\hline
\end{tabular}
\caption{
 AUC \& Unfairness over 10 repetitions. Colored values highlight the achieved fairness.
}
\label{tab:results_sims}
\end{center}

\end{table}

The numerical results are summarized in Table \ref{tab:results_sims}. The \emph{Uncalibrated} columns present the results for a standard model that does not have a Calibration Layer, the \emph{Fairness-aware} columns represent a model of the form of Figure \ref{fig:model_arch}. As expected, all uncalibrated models present a significant level of unfairness, indicating the model learned to use gender in its predictions. The fairness-aware architecture manages to eliminate the bias, as indicated in the highlighted column in the results table, though it also results in a lower predictive accuracy as suggested by the theoretical analysis. 

We then apply our XAI methodology to the data. We use the \emph{Bias size} task from Table \ref{tab:BiasDetectTask}, re-fit the model to the new task $\Tilde{Y}^\tau$ and evaluate the attention maps using the Grad-CAM algorithm of \cite{selvaraju2017grad}. To see how this can effectively prevent \emph{fairwashing} and help establish a causal relation, we compare the results to attention maps obtained from the initial model (that is, from the model that modelled the \emph{Beard} task). Results are visualized in Figure \ref{fig:gradcam_all}. The left three columns are averages of 5,000 well-aligned raw images, split by the \emph{Gender} variable and the Grad-CAM attention map of the initial model. Most of the attention is indeed focused around the region where one would expect the relevant characteristics of a beard to lie, indicating a small bias (second and third column). However, this would stand in stark contrast with the numerical results in Table \ref{tab:results_sims}. With the help of our methodology we can instead isolate the regions that contribute to the bias which we observe (fourth column). We can clearly see that the neural network extensively focuses on features associated with gender (such as a receding hairline, twice as likely for male individuals or blond hair, around 10 times more likely for female individuals, or earrings around 20 times more likely for females), but not the area where a beard would be expected. Such ratios are easily identifiable and can be investigated further and our method can effectively help identify relevant features.

\begin{figure*}[htbp]
  \centering
  \includegraphics[width=0.7\textwidth]{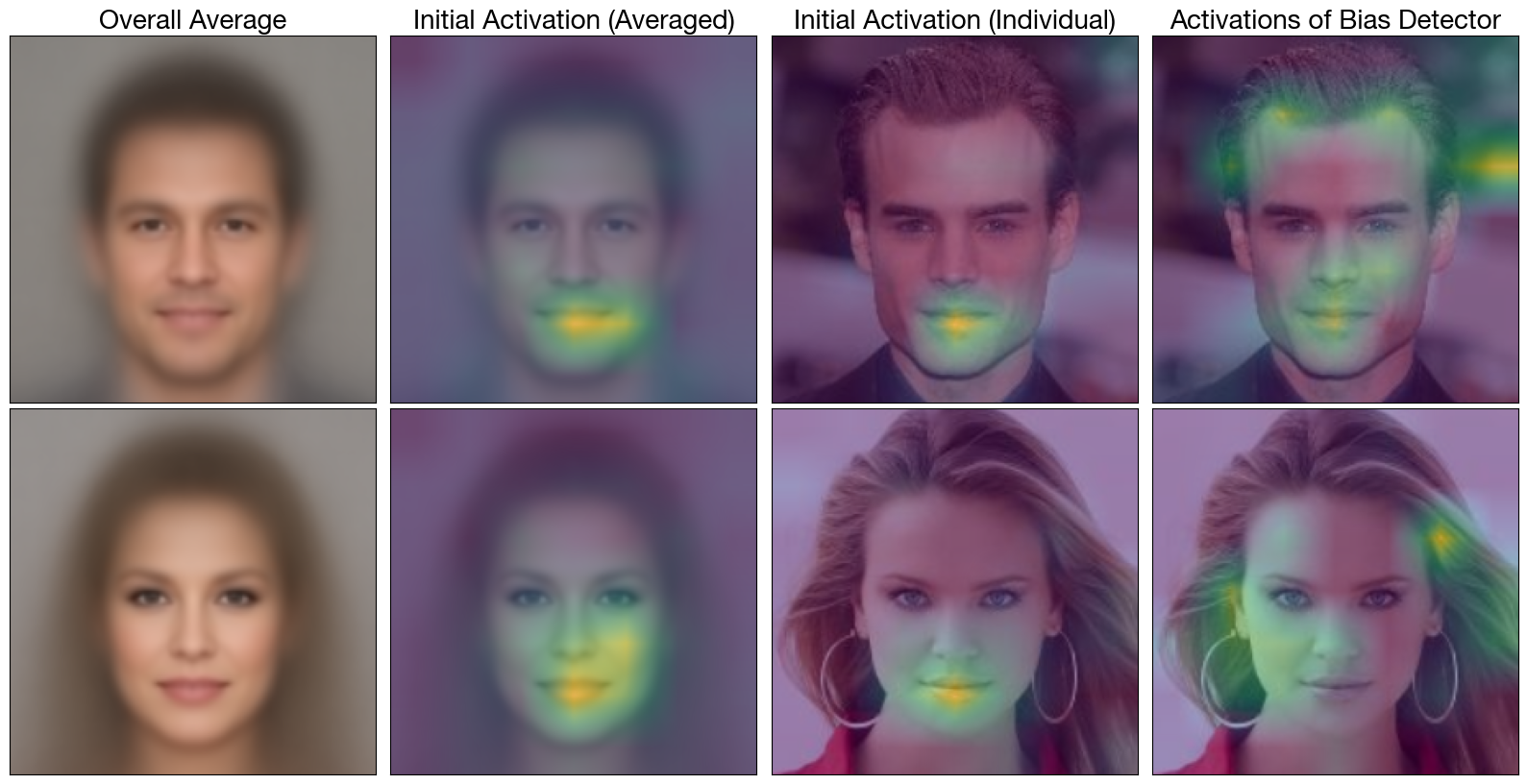}  \caption{GradCAM activations with gender as a sensitive feature when estimating ``Beard" task.}
  \label{fig:gradcam_all}
\end{figure*}

\section{Conclusion}

We investigated the use of optimal transport in both fairness and explainability applications for machine learning methods. Though both fields have produced significant advanced in recent years, the strong ties between the two fields remain underexplored. We showed through our applications that a standard optimal transport method for tabular data can readily be extended to computer vision tasks, by adapting the model architecture slightly. This resulted in an extremely efficient routine that can easily be integrated into standard workflows. Given that the method is based on optimal transport, we were then able to derive an XAI method based on the optimal transportation plan, which helps to identify the sources of a bias, permitting a more targeted investigation into the causes rather than the symptoms. This method also enables researchers to adopt a more holistic approach to the choice of sensitive variables, alleviating concerns of \emph{fairwashing} \cite{aivodji2019fairwashing} and indeed opens up a more objective discussion about emerging issues related to intersectional - \cite{kong2022intersectionally} or sequential \cite{hu2023sequentially} fairness.


\acks{This research was enabled in part by support provided by Calcul Québec (calculquebec.ca) and the Digital Research Alliance of Canada (alliancecan.ca)}

\bibliography{bibliography}

\appendix

\section{Visual representation}\label{app:vizref}

To help with the interpretation of the different quantities in Section \ref{sec:XAIImage}, we created a graph below. As a simple working example, consider a wage model and we would like to analyse the distributions induced by a binary sensitive feature. In Figure \ref{fig:explainer_appendix}, we visualise these two marginal distributions, using their CDF. The CDF has the advantage that it scales both marginal distributions to the same range. 

Consider predictions run on the x-axis (from $f^*(x)$, where the sensitive feature is implicitly included in the features). As explained in the text, we would like to compare two individuals across their relative position within their subgroup (recall that, intuitively, the barycenter we apply to achieve fairness recreates a common distribution - but keeps the relative within group ordering constant to minimize the transportation cost). This can be done via the projection, that is, for a value on the x-axis, say the value obtained for $f(x_0, s=1)$, we calculate the value on the CDF for the group $s=1$ (the blue curve) and project the point to the CDF of the group $s=2$ (the orange dotted line). This then gives us the counterfactual based on optimal transport on the $x-axis$. Compare this to the ceteris paribus prediction (that is $f(x_0, s=2)$), which does not take into account the group level changes in $x_0$. This allows a very natural interpretation of the distance we are using for our auxiliary task. In a sense, we would like to explain the features that are relevant to explain large differences, be it directly related (that is the distance $|f(x_0,s=1) - f(x_0,s=2)|$) or indirectly related (the remainder of the total distance) to the group variable $s$.

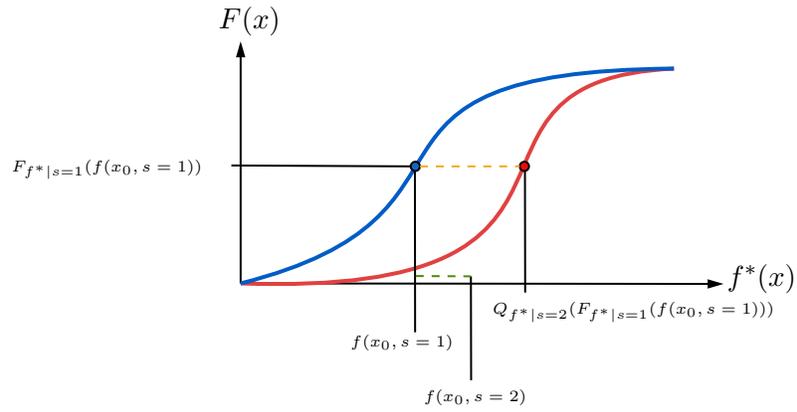
\begin{figure}[ht]
	\vspace{.1in}
	\centering
    \tikzset{every picture/.style={line width=0.75pt}} 

\begin{tikzpicture}[x=0.75pt,y=0.75pt,yscale=-1,xscale=1]

\draw    (101.01,59.1) -- (101.01,179.5) ;
\draw    (101.01,179.5) -- (339.93,179.77) ;
\draw [color={rgb, 255:red, 224; green, 65; blue, 65 }  ,draw opacity=1 ][line width=1.5]   (101.01,179.5) .. controls (306.6,186.43) and (191.6,77.1) .. (319.54,71.09) ;
\draw  [color={rgb, 255:red, 0; green, 0; blue, 0 }  ,draw opacity=1 ][fill={rgb, 255:red, 239; green, 10; blue, 10 }  ,fill opacity=1 ][line width=0.75]  (241.77,120.43) .. controls (241.77,119.14) and (242.81,118.1) .. (244.1,118.1) .. controls (245.39,118.1) and (246.43,119.14) .. (246.43,120.43) .. controls (246.43,121.72) and (245.39,122.77) .. (244.1,122.77) .. controls (242.81,122.77) and (241.77,121.72) .. (241.77,120.43) -- cycle ;
\draw [color={rgb, 255:red, 2; green, 92; blue, 198 }  ,draw opacity=1 ] [line width=1.5]  (101.01,179.5) .. controls (239.6,144.43) and (137.93,72.43) .. (319.54,71.09) ;
\draw  [color={rgb, 255:red, 0; green, 0; blue, 0 }  ,draw opacity=1 ][fill={rgb, 255:red, 23; green, 101; blue, 193 }  ,fill opacity=1 ][line width=0.75]  (186.77,120.43) .. controls (186.77,119.14) and (187.81,118.1) .. (189.1,118.1) .. controls (190.39,118.1) and (191.43,119.14) .. (191.43,120.43) .. controls (191.43,121.72) and (190.39,122.77) .. (189.1,122.77) .. controls (187.81,122.77) and (186.77,121.72) .. (186.77,120.43) -- cycle ;
\draw    (244.39,122.77) -- (244.5,184.13) ;
\draw    (96.35,119.98) -- (186.77,120.43) ;
\draw    (188.89,122.77) -- (189,204.2) ;
\draw  [fill={rgb, 255:red, 0; green, 0; blue, 0 }  ,fill opacity=1 ] (101.01,59.1) -- (102.71,64.9) -- (99.31,64.9) -- cycle ;
\draw  [fill={rgb, 255:red, 0; green, 0; blue, 0 }  ,fill opacity=1 ] (342.83,179.77) -- (337.03,181.47) -- (337.03,178.07) -- cycle ;
\draw[dashed] [color={rgb, 255:red, 245; green, 166; blue, 35 }  ,draw opacity=1 ]   (191.43,120.43) -- (241.77,120.43) ;
\draw    (217.14,175.94) -- (217.25,228.38) ;
\draw[dashed] [color={rgb, 255:red, 81; green, 137; blue, 11 }  ,draw opacity=1 ]   (189.35,175.78) -- (217.14,175.94) ;

\draw (88.52,38.3) node [anchor=north west][inner sep=0.75pt]    {$F( x)$};
\draw (345,168.11) node [anchor=north west][inner sep=0.75pt]    {$f^{*}( x)$};
\draw (226.72,187.38) node [anchor=north west][inner sep=0.75pt]  [font=\tiny]  {$Q_{f^{*} |s=2}( F_{f^{*} |s=1}( f( x_{0} ,s=1)))$};
\draw (155.1,203.93) node [anchor=north west][inner sep=0.75pt]  [font=\tiny]  {$f(x_{0},s=1)$};
\draw (-15.97,115.4) node [anchor=north west][inner sep=0.75pt]  [font=\tiny]  {$F_{f^{*} |s=1}( f( x_{0} ,s=1))$};
\draw (192.05,231.53) node [anchor=north west][inner sep=0.75pt]  [font=\tiny]  {$f( x_{0} ,s=2)$};

\end{tikzpicture}
    \caption{
    Illustration of different points and counterfactuals. 
    }
    \label{fig:explainer_appendix}
\end{figure}

\end{document}